\documentclass[]{interact}

\usepackage{epstopdf}
\usepackage[caption=false]{subfig}
\usepackage{cite}
\usepackage{amsmath,amssymb,amsfonts}
\usepackage{algorithmic}
\usepackage{graphicx}
\usepackage{textcomp}
\usepackage{xcolor}
\usepackage[linesnumbered]{algorithm2e}
\usepackage{siunitx}

\usepackage[numbers,sort&compress]{natbib}
\bibpunct[, ]{[}{]}{,}{n}{,}{,}

\theoremstyle{plain}
\newtheorem{theorem}{Theorem}[section]

\newtheorem{proposition}[theorem]{Proposition}

\theoremstyle{definition}

\theoremstyle{remark}

\begin{document}


\title{Scale space radon transform-based inertia axis and object central symmetry estimation}

\author{
\name{Aicha Baya Goumeidane\textsuperscript{a}\thanks{CONTACT Aicha Baya Goumeidane Email: a.goumeidane@crti.dz}, Djemel Ziou\textsuperscript{b}, and Nafaa Nacereddine\textsuperscript{a} }
\affil{\textsuperscript{a}Research Center in industrial Technologies, CRTI, P.O.BOX 64, Algiers, Algeria; \textsuperscript{b}Département informatique, Université de Sherbrooke, Québec, Canada}
}

\maketitle

\begin{abstract}
	Inertia Axes are involved in many techniques for image content measurement when involving information obtained from lines, angles, centroids... etc. We investigate, here, the estimation of the main axis of inertia of an object in the image. We identify the coincidence conditions of the Scale Space
Radon Transform (SSRT) maximum and the inertia main axis. We show, that by choosing the appropriate scale parameter, it is possible to match the SSRT maximum and the main axis of inertia location and orientation of the embedded object in the image. Furthermore, an example of use case is presented where binary objects central symmetry computation is derived by means of SSRT projections and the axis of inertia orientation. To this end, some SSRT characteristics have been highlighted and exploited. The experimentations show the SSRT-based main axis of inertia computation effectiveness. Concerning the central symmetry, results are very satisfying as experimentations carried out on randomly created images dataset and existing datasets have permitted to divide successfully these images bases into centrally symmetric and non-centrally symmetric objects. 
\end{abstract}


\section{Introduction}

Computing the main axis of inertia can be very useful to characterize some aspects of the object. Indeed, it has been involved in many works focusing on symmetry evaluation like the works presented in~\cite{Mara,Tuzikov1, Gothandaraman,Tuzikov2} or for pedestrian detection and for real time estimation of body posture purposes in \cite{Fang} and in \cite{Iwasawa},  respectively. In addition, the authors in \cite{Fan} and in \cite{Malek}, have involved the main axis of inertia computation in medical applications for cross-sectional clinical images analysis purpose and for vessel centerline extraction in retinal image one. In industrial applications, computing axis of inertia has been used, for example, for image processing-based automatic monitoring of industrial process in \cite{Qin} and for non destructive testing of wood in \cite{Long}. Other works proposed in \cite{Mahdikhanlou,Bedini,Zhang} have made use of the axis of inertia to the aim of leaf plant description, ship characterization or finger print location. 

Recently, authors in \cite{Ziou} have proposed a transform called the Scale Space Radon Transform (SSRT), which can be viewed as a significant generalized form of the Radon transform. They have shown that this transform can be used to detect elegantly and accurately thick lines and ellipses through an embedded kernel tuned by a scale space parameter. When the scale space parameter is an appropriate one, the maximum of SSRT represents the centerlines of the linear/elliptical structures presented in the image \cite{Goumeidane1,Goumeidane2}. In this paper we propose to investigate the ability of the SSRT to provide, through its maximum in the SSRT space, the main axis of inertia of the processed object. Moreover, on the basis of some SSRT characteristics and of the axis of inertia orientation, as application, a method to check central symmetry of binary objects is presented.

The remainder of this paper is organized as follows.
In Sect.2, the SSRT and the main axis of inertia computation via geometric moments are introduced. In Sect.3, the relationship between the SSRT maximum and the geometric moments-based main axis of inertia parameters is highlighted. in Sect.4 Computing object symmetry with respect to a point with the SSRT is presented on the basis of some SSRT properties that are demonstrated. Experimentations and results are provided in Sect.5. We finish the paper by drawing the main conclusions. 

\section{Methods and material}
\subsection{Scale Space Radon Transform}
The Scale Space Radon Transform (SSRT), $\check{S_f}$, of an image $f$, is a matching of a kernel and an embedded parametric shape in this image. If the parametric shape is a line parametrized by the location parameter $\rho$ and the angle $\theta$ and the kernel is a Gaussian one, then $\check{S_f}$ is given by \cite{Ziou}
\begin{equation}\label{ssrt}
\check{S_f}(\rho,\theta,\sigma)=\frac{1}{\sqrt{2\pi}\sigma} \int_{\cal X} \int_{\cal Y} f(x,y)e^{-\frac{(x\cos \theta +y\sin \theta -\rho)^2}{2\sigma^2}}dxdy
\end{equation}
Here, $\sigma$ is the scale space parameter. It returns that the Radon Transform (RT) $\check{R_f}$ \cite{N} is a special case of $\check{S_f}$ (when $\sigma\rightarrow 0$)\cite{Ziou}, as the matching between an embedded thin structure in an image and the Dirac distribution $\delta$ of an implicit parametric shape in $\check{R_f}$, is replaced by a matching between an embedded parametric shape of arbitrary thickness in an image and a kernel tuned by a scale parameter in $\check{S_f}$. 
In fact, such replacement allows the transform, through the kernel, to handle correctly embedded shapes even if they are not filiform, unlike $\check{R_f}$. The role of the chosen kernel is to control the parametric shape position inside the embedded object via the scale parameter and then, the detection is reduced to maxima detection in the SSRT space. 
\begin{figure}[h]
	\centering
	\includegraphics[width=80mm]{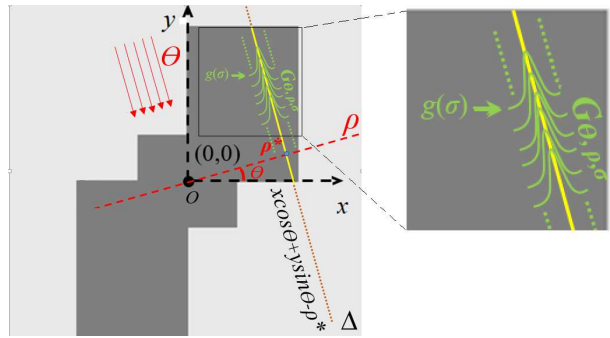}
	\caption{Directional Gaussian $G_{\theta,\rho,\sigma}$ used in SSRT computation and the unidimensional Gaussians $g(\sigma)$. }\label{Fig}	
\end{figure}
 The directional Gaussian in \eqref{ssrt}, $G_{\theta,\rho,\sigma}(x,y)=\frac{1}{\sqrt{2\pi}\sigma}e^{-(x\cos\theta+y\sin\theta-\rho)^2/2\sigma^2}$, can be viewed in Fig.1, where the cross section of it is the Gaussian $g(\sigma)=\frac{1}{\sqrt{2\pi}\sigma} e^{- (x\cos\theta+y\sin\theta-\mu_\rho)^2/2\sigma^2}$ of which mean value $\mu_\rho$ belongs to the line $\Delta$ with equation $ x\cos\theta +y\sin\theta -\rho =0$, as seen in the same figure. Furthermore, the authors in \cite{Nacer} have shown that there is a relationship between RT and SSRT, expressed as follows
 \begin{equation}\label{ssrt_conv}
\check{S_f}=\check{R_f}\circledast g(\sigma).
 \end{equation}
Where $\circledast$ is the convolution symbol. The formulae \eqref{ssrt_conv} constitutes a straightforward way to compute the SSRT which is the convolution of the RT with the 1d kernel $g(\sigma)$.

\subsection{Main axis of inertia computation}
The main axis of inertia of an object embedded in an image $f$ is the minimum-inertia
line of the slope $\phi$ passing through the centroid of this object, of which coordinates are $x_c$ and $y_c$ \cite{Zunic}. The direction $\phi$ can be expressed with second order centred geometric moments $\mu_{11}$, $\mu_{20}$ and $ \mu_{02}$, and is given by \cite{Teague}
\begin{equation}\label{phi}
\tan 2\phi=\frac{2\mu_{11}}{\mu_{20}-\mu_{02}}
\end{equation}
The centred geometric moments have as expressions $\mu_{ij}=\int_x\int_y (x-x_c)^i(y-y_c)^jf(x,y)dxdy$. Moreover, $\mu_{11}$, $\mu_{20}$ and $ \mu_{02}$ can be easily expressed through non-centred ones $m_{11}$, $m_{02}$, $m_{20}$, $m_{01}$, $m_{10}$, where $m_{ij}=\int_x\int_y x^iy^jf(x,y)dxdy$. Furthermore, to have the simplest expressions, instead of using $f(x,y)$ we will use $\frac{f(x,y)}{\int_x\int_y f(x,y) dxdy}$ and rename it $f(x,y)$. Because the image $f$ is non negative and because of the normalization, we have$\quad$ $0\leq f(x,y)\leq 1$. Subsequently, we have 
\begin{equation}\label{m00}
\int_x\int_y f(x,y)dxdy=m_{00}=1
\end{equation} 
Formulas of $\mu_{ij}$ through $m_{ij}$ are given by
\begin{equation}\label{mu_with_m}
\mu_{11}=m_{11}-m_{01}m_{10},\ \mu_{02}=m_{02}-m_{01}^2, \ \mu_{20}=m_{20}-m_{10}^2
\end{equation}
On the other hand, the centroid coordinates are computed as
\begin{equation}\label{xcycmoments}
x_c=\frac{\int_x\int_y x f(x,y)dxdy}{\int_x\int_y f(x,y)dxdy}, \ y_c=\frac{\int_x\int_y y f(x,y)dxdy}{\int_x\int_y f(x,y)dxdy}
\end{equation}
Replacing \eqref{mu_with_m} in \eqref{phi} yields the orientation expression with non-centred moments. Then, if $m_{20}- m_{02}+m_{01}^2 -m_{10}^2 \ne 0 $
\begin{equation}\label{phi_bis}
\tan 2\phi=\frac{2(m_{11}-m_{01}m_{10} )}{m_{20}- m_{02}+m_{01}^2 -m_{10}^2}
\end{equation}
If $m_{20}- m_{02}+m_{01}^2 -m_{10}^2 = 0 $ then $ \phi= \pi/4$ or $ \phi= -\pi/4$ depending on the sign of $m_{11}-m_{01}m_{10}$, as shown in \cite{Teague}. 

\section{The Scale Space Radon Transform maximum and the main inertia axis}
We show in the following that the SSRT maximum provides the main axis of inertia. So, to this end, we need to find the line parameters $\hat{\theta}$ and $\hat{\rho}$ maximizing $\check{S_f}$ in \eqref{ssrt}. Under the continuity assumption of $\check{S_f}$, $\hat{\theta}$ and $\hat{\rho}$ should be the solution of first order derivatives $\frac{\partial \check{S_f} }{\partial\theta}=0$ and $\frac{\partial \check{S_f} }{\partial\rho}=0$, respectively. Let begin by $\theta$
\begin{equation}\label{derivS_teth}
\frac{\partial \check{S_f} }{\partial\theta}=  \int_{\cal X} \int_{\cal Y}f(x,y) \frac{\partial e^{-(x\cos \theta +y\sin \theta -\rho)^2/2{\sigma}^2}}{\partial\theta } dxdy=0
\end{equation}
Equation \eqref{derivS_teth} is highly non linear in $\theta$ and $\rho$. However, the region of interest of this function is around the line $z=x \cos\theta +y \sin \theta-\rho=0$. We propose, then, an approximation of the function $g(z)=e^{-z^2/2\sigma^2}$ around z=0 by the Maclaurin series,
\begin{equation}\label{Maclaurin}
g(z)\simeq\sum^\infty_{n=0}(-1)^n\frac{1}{n!}\left(\frac{z^2}{2\sigma^2}\right)^n=1+\sum^\infty_{n=1}(-1)^n\frac{1}{n!}\left(\frac{z^2}{2\sigma^2}\right)^n
\end{equation}
We prove in the following the convergence of the proposed approximation. The parameter $\sigma$ being always greater than 1, as we will see later, the convergence of the serie in \eqref{Maclaurin} is proven by the alternating series test, where $g(z)\simeq\sum^\infty_{n=0}(-1)^n a_n$, as follows 
\begin{enumerate}
\item  As $z\sim 0$ since the approximation is done around $z=0$ and $\sigma > 1$, then, consequently $\lim_{n \to +\infty} \frac{1}{n!}\left(\frac{z^2}{2\sigma^2}\right)^n=0$.
\item $a_{n+1}=\frac{1}{n+1!}(\frac{z^2}{2\sigma^2})^{n+1}=\frac{1}{n!}(\frac{z^2}{2\sigma^2})^n\frac{1}{n+1}\frac{z^2}{2\sigma^2}=\frac{z^2}{(n+1)2\sigma^2}a_n< a_n$, $\forall n\geq0$.
\end{enumerate} 
From 1 and 2, the serie in \eqref{Maclaurin} converges for $z<1$ and $\frac{z}{\sqrt{2}\sigma}<1$. Consequently, the derivative can be approximated and rewritten as follows 
\begin{equation}\label{derv_g_thet}
\frac{\partial g(z)}{\partial \theta}\simeq\sum^N_{n=1}(-1)^n\frac{1}{n-1!}\frac{2}{(2\sigma^2)^n} z^{2n-1}\frac{\partial z}{\partial \theta}\quad \text{and}\quad\frac{\partial z}{\partial \theta}=-x\sin\theta+y\cos\theta 
\end{equation}
For $N$ equals to 1 ($N=1$) and using \eqref{derv_g_thet}, \eqref{derivS_teth} becomes 
\begin{equation}\label{deriv_S_thet_simplif1}
\int_{\cal X} \int_{\cal Y}f(x,y) \left (x\cos\theta+y\sin\theta-\rho\right) \left(-x\sin\theta+y\cos\theta\right)dxdy=0
\end{equation}
By developing we obtain
\begin{equation} \label{deriv_S_thet_simplif2}
\begin{split}
\int_{\cal X} \int_{\cal Y}f(x,y) \left[-x^2\cos\theta\sin\theta+xy\cos^2\theta-xy\sin^2\theta \right.
\\
\left.+y^2\sin\theta\cos\theta -\rho\left(-x\sin\theta+y\cos\theta\right)\right] dxdy=0
\end{split}
\end{equation}
And then
\begin{equation} \label{deriv_S_thet_simplif3}
\int_{\cal X} \int_{\cal Y}f(x,y) (-x^2(\sin 2\theta)/2+y^2(\sin 2\theta)/2+ xy\cos 2\theta
-\rho(-x\sin\theta+y\cos\theta))dxdy=0
\end{equation}
 Rewriting this equation using the geometrics moments yiels 
\begin{equation}\label{deriv_S_thet_with_m}
-m_{20}(\sin 2\theta)/2+m_{02}(\sin 2\theta)/2+ m_{11}\cos 2\theta
-\int_{\cal X} \int_{\cal Y}f(x,y)\rho(-x\sin\theta+y\cos\theta)dxdy=0
\end{equation}

Consider now the first order condition $\frac{\partial F}{\partial\rho}$.
Using again the Maclaurin development around $z=0$  we have 
\begin{equation}\label{deriv_g_ro}
\frac{\partial g(z)}{\partial \rho}\simeq\sum^N_{n=1}(-1)^n\frac{1}{n-1!}\frac{2}{(2\sigma^2)^n} z^{2n-1}\frac{\partial z}{\partial \rho} \quad \text{where} \quad \frac{\partial z}{\partial \rho}=-1 
\end{equation}
Setting $\frac{\partial g(z)}{\partial \rho}=0$ and taking $N=1$ in \eqref{deriv_g_ro}, leads to the following equation
\begin{equation}\label{dervi_g_ro_simplif1}
\int_{\cal X} \int_{\cal Y}f(x,y)(x\cos\theta+y\sin\theta-\rho)dxdy=0 
\end{equation} 
Then $\rho = \frac{\int_{\cal X} \int_{\cal Y}xf(x,y)dxdy}{\int_{\cal X} \int_{\cal Y}f(x,y)dxdy}\cos\theta+ \frac{\int_{\cal X} \int_{\cal Y}yf(x,y) dxdy}{\int_{\cal X} \int_{\cal Y}f(x,y)dxy}\sin\theta$ and is rewritten using \eqref{xcycmoments} as:
\begin{equation}\label{ro_with_xcyc}
\rho=x_c\cos\theta+y_c\sin\theta 
\end{equation}
Involving \eqref{m00} and the non-centred moments expressions, \eqref{ro_with_xcyc} becomes 
\begin{equation}\label{ro_thet_m}
\rho=m_{10}\cos\theta+m_{01}\sin\theta 
\end{equation}
Replacing $\rho $ by its expression in \eqref{deriv_S_thet_with_m} yields 
\begin{equation}\label{deriv_thet_ro}
\begin{split}
-m_{20}(\sin 2\theta)/2+m_{02}(\sin 2\theta)/2+m_{11}\cos 2\theta-\int_{\cal X} \int_{\cal Y}f(x,y)\left[\left(m_{10}\cos\theta+m_{01}\sin\theta\right)\right. \\\left.\left(-x\sin\theta+y\cos\theta\right)\right]dxdy=0\qquad\qquad\qquad\qquad\qquad\qquad
\end{split}
\end{equation}
Contracting \eqref{deriv_thet_ro} provides
\begin{equation}
\frac{1}{2}\sin 2\theta\left(-m_{20}+ m_{02}-m_{01}^2\right.  \left.+m_{10}^2\right) + \cos 2\theta\left(m_{11}-m_{01}m_{10} \right)=0\
\end{equation}
Then, if $m_{20}- m_{02}+m_{01}^2 -m_{10}^2 \ne 0 $ we can write
\begin{equation}\label{thet_as_phi}
\tan 2\hat{\theta}=\frac{2(m_{11}-m_{01}m_{10} )}{m_{20}- m_{02}+m_{01}^2 -m_{10}^2}
\end{equation}
If $m_{20}- m_{02}+m_{01}^2 -m_{10}^2=0$, then $\theta=\pi/4$ or $\theta=-\pi/4$ depending on the sign of the numerator.

For the following and before going further, let us go back to \eqref{ro_with_xcyc} to estimate $\hat{\rho}$ the location parameter that maximizes $\check{S_f}$ and then derive a necessary relationship between $\hat{\rho}$ and $\hat{\theta}$. So, according to \eqref{ro_with_xcyc}, any critical line $\rho=x\cos\theta+y\sin\theta$ maximizing $\check{S_f}$ passes through the point $(x_c,y_c)$, whatever the angle $\theta$. When this angle is $\hat{\theta}$, the corresponding $\rho$ i.e. $\hat{\rho}$ is given by $\hat{\rho}=x_c\cos\hat{\theta}+y_c\sin\hat{\theta}$, which means that $(x_c,y_c)$ verifies the equation of the line 
\begin{equation}\label{hatro_hat_thetxy}
\hat{\rho}=x\cos\hat{\theta}+y\sin\hat{\theta}
\end{equation} 
Consequently, derivating the SSRT with respect to $\theta$ and $\rho$ provides, in the image domain, a line of which equation is represented by \eqref{hatro_hat_thetxy}.

Now, before showing the relationship between the orientation $\phi$ in \eqref{phi} and the angle $\hat\theta$ in \eqref{thet_as_phi}, and to be in accordance with our proposition which states that the SSRT maximum defines the inertia principal axis, we must prove that the computed critical point ($\hat{\theta},\hat{\rho}$) is a maximum. To this end, let us compute $H$ the Hessian of $\check{S_f}$.
\begin{equation}\label{Hessian}
H=\begin{pmatrix} H_{11}&H_{12}\\ H_{21}&H_{22} \\ \end{pmatrix}=\begin{pmatrix}\frac{\partial^2 }{\partial \theta^2}\check{S_f}& \frac{\partial^2 }{\partial \theta \partial \rho}\check{S_f}\\ \frac{\partial^2 }{\partial \rho\partial \theta}\check{S_f}& \frac{\partial^2 }{\partial \rho^2}\check{S_f} \\ \end{pmatrix} \normalsize
\end{equation}
To be a maximum, the critical point ($\hat{\theta}$,$\hat{\rho}$) must verify two conditions: 1) $H_{11}$ at ($\hat{\theta}$,$\hat{\rho}$)  is negative, 2) the determinant of $H$, det($H$) at ($\hat{\theta}$,$\hat{\rho}$) is positive. So let us write $\check{S_f}$ as $\check{S_f}(\rho,\theta)=\frac{1}{\sqrt{2\pi} \sigma}\int_{\cal X} \int_{\cal Y}f(x,y)e^{-B}dxdy$, where $B=(x\cos \theta+y\sin \theta-\rho)^2/2{\sigma}^2$. Hence, let us compute the second order partial derivatives at the critical point $(\hat{\theta}, \hat{\rho})$.
\begin{enumerate}
 \item Computing $\frac{\partial^2 \check{S_f}}{\partial \theta^2}$ at the critical point \small
\begin{equation}
\begin{split}
\frac{\partial^2 \check{S_f}}{\partial \theta^2}=\frac{-2}{\sqrt{2\pi}\sigma }\int_{\cal X}\int_{\cal Y} f(x,y) e^{-B}\left[-2\left(x\cos\theta+y\sin\theta-\rho\right)^2 
\left(-x\sin\theta+y\cos\theta\right)^2+\right.
\\\left.\left(-x\sin\theta+y\cos\theta\right)^2-
\left(x\cos\theta+y\sin\theta-\rho\right)\left(x\cos\theta+y\sin\theta\right)\right]dxdy\qquad\qquad
\end{split}
\end{equation}
\normalsize
At the critical point, we have $\hat{\rho}=x\cos\hat{\theta}+y\sin\hat{\theta}$ and then B=0. Consequently,
\begin{equation}
\frac{\partial^2 \check{S_f}}{\partial \theta^2}=\frac{-2}{\sqrt{2\pi}\sigma }\int_{\cal X}\int_{\cal Y} f(x,y)(-x\sin\hat{\theta}+y\cos\hat{\theta})^2 dxdy <0
\end{equation}

 \item Computing $\frac{\partial^2 \check{S_f}}{\partial \rho^2}$ at the critical point 
\begin{equation}
\frac{\partial^2 \check{S_f}}{\partial \rho^2}=\frac{2}{\sqrt{2\pi}\sigma }\int_{\cal X}\int_{\cal Y} f(x,y) e^{-B}(2(x\cos\theta+y\sin\theta-\rho)^2-1)dxdy
\end{equation}
At the critical point 
\begin{equation}
\frac{\partial^2 \check{S_f}}{\partial \rho^2}=\frac{-2}{\sqrt{2\pi}\sigma }\int_{\cal X}\int_{\cal Y} f(x,y)dxdy=\frac{-2}{\sqrt{2\pi}\sigma }
\end{equation}

 \item Computing $\frac{\partial^2 \check{S_f}}{\partial\theta\partial\rho}$ at the critical point 
\begin{equation}
\begin{split}
\frac{\partial^2\check{S_f}}{\partial\theta\partial\rho}=\frac{2}{\sqrt{2\pi}\sigma }\int_{\cal X}\int_{\cal Y} f(x,y) e^{-B}(-2(x\cos\theta+y\sin\theta-\rho)(x\sin\theta+y\cos\theta)\\+(-x\sin\theta+y\cos\theta))dxdy\qquad\qquad\qquad\qquad
\end{split}
\end{equation}
At the critical point 
\begin{equation}
\frac{\partial^2 \check{S_f}}{\partial\theta\partial\rho}=\frac{2}{\sqrt{2\pi}\sigma }\int_{\cal X}\int_{\cal Y} f(x,y) (-x\sin\hat{\theta}+y\cos\hat{\theta})dxdy
\end{equation}
\end{enumerate}
As $\check{S_f}$ is a $C^2$ function, then the mixed partial derivatives in \eqref{Hessian} are the same. It follows that the determinant $det(H)$ equals to 
\begin{equation}\label{Determinant}
\begin{split}
det(H)=\frac{4}{2\pi\sigma}\left[\int_{\cal X}\int_{\cal Y} f(x,y) (-x\sin\hat{\theta}+y\cos\hat{\theta})^2dxdy- \right.\\ \left.
\left(\int_{\cal X}\int_{\cal Y} f(x,y) (-x\sin\hat{\theta}+y\cos\hat{\theta})dxdy\right)^2\right]
\end{split}
\end{equation}
We recall that $0\leq f(x,y)\leq 1$ and $\int \int f(x,y)dxdy=1$, then $f(x,y)$ can be considered as a probability density function of two variables and $v(x,y)=(-x\sin\hat{\theta}+y\cos\hat{\theta})$ a 2D random variable. Thus, $det(H)$ in \eqref{Determinant} is a variance up to a positive constant, which makes it non-negative. However, in our case this variance could not be null as $v(x,y)$ is not constant. Then $det(H)>0$. Consequently, $H_{11} <0$ and det$(H) >0 $ for the critical point ($\hat{\theta},\hat{\rho}$), which makes it a maximum.
\begin{figure}[h]
	\centering
	\includegraphics[width=70mm]{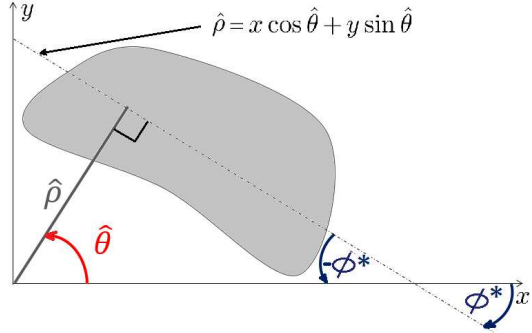}
	\caption{Relationship between the angle $\hat{\theta}$  and the orientation angle $\phi^*$}\label{Fig}	
\end{figure}

To finish, as shown in Fig.2, the relationship between the angle $\hat{\theta}$ that maximizes $\check{S_f}$ and the orientation angle $\phi^*$ of the object under investigation is given  by: $\phi^*=\hat{\theta}-\pi/2$. Let us now compute $\tan 2\phi^*$ as:  $\tan 2\phi^*=\tan(2\hat{\theta}-\pi)=\frac{\tan2\hat{\theta}-\tan\pi}{1+\tan2\hat{\theta}\tan \pi}=\tan2\hat{\theta}$. Thus 
\begin{equation}\label{phi_aste}
\tan 2\phi^*=\frac{2(m_{11}-m_{01}m_{10} )}{m_{20}- m_{02}+m_{01}^2 -m_{10}^2}
\end{equation}
This means that the angle $\phi$ computed with the geometric moments in \eqref{phi_bis} and $\phi^*$ the one computed after maximizing $\check{S_f}$, are the same. Indeed, the maximum of $\check{S_f}$ is a line of slop $\phi^*=\phi$ passing through $(x_c,y_c)$ as it is the case of the inertia main axis. It follows that there is an exact match of these two lines. 

Lastly, it is worth to note that the spatial coincidence between the line maximizing $\check{S_f}$ and the geometric moments-based principal axis holds only when the application of the SSRT provides one maximum. For this reason, $\sigma$ must be chosen to meet the aforementioned requirement. The easiest way to fulfil this condition, is to set $\sigma$ to the greatest distance separating two points of the binary pattern under investigation. We can see in Fig.3 the SSRT and its maxima. We can notice, in Fig.3.a, that when the parameter $\sigma$ is an appropriate one, there is only one SSRT maximum and one line corresponding to this maximum. When $\sigma$ is set to a smaller value, the SSRT space in Fig.b, shows three maxima, but none of them corresponds to the axis of inertia. Moreover, the error induced by the approximation used to prove the previously mentioned spatial coincidence is evaluated in the appendix.
\begin{figure}[h]
	\centering
	\includegraphics[width=145mm]{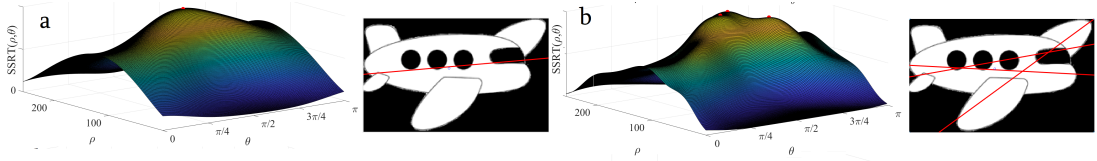}
	\caption{Example of SSRT maxima obtained with different scale parameters $\sigma$, where the SSRT exhibits one maximum (a) and  three maxima with smaller and inadequate scale parameter in (b). }\label{Fig}	
\end{figure}

\section{Example of application : Measuring central symmetry of binary objects}
The central symmetry is the symmetry with respect to a point and it is equivalent to a rotational symmetry of 2 folds \cite{Pei}. So, a binary object is said to be centrally symmetric if its rotated version around its centroid by $\pi$ is identical to it\cite{Pei}. It results that objects with $2n$ folds rotational symmetry, where $n\in \mathbb{N}$, are centrally symmetric. Let $f$ be a centrally symmetric binary image having its center as axes origin $O(0,0)$. It results that if a point of coordinates $(x,y)$ is rotated by $\pi$ around $O(0,0)$ then, its new coordinates are \small $\begin{pmatrix}
\cos\pi &-\sin\pi\\
\sin\pi&\cos\pi  
\end{pmatrix}$ $\begin{pmatrix}
x\\
y 
\end{pmatrix}$ =$\begin{pmatrix}
-x\\
-y 
\end{pmatrix}$\normalsize. Applying the central symmetry definition previously stated leads to $f(x,y)=f(-x,-y)$.
 \begin{figure}[h]
	\centering
	\includegraphics[width=110mm]{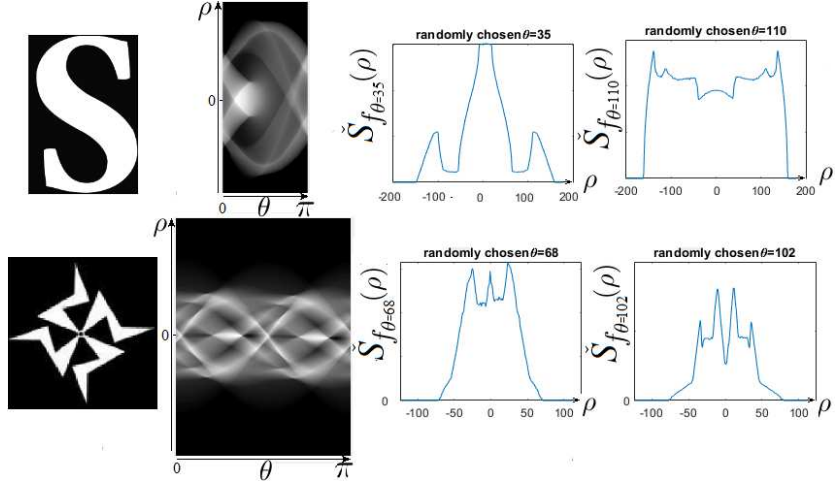}
	\caption{From left to right: Two centrally symmetric objects centred in the image, their SSRT space, two randomly chosen directions ($\theta$) for each case and their corresponding SSRT projections $\check{S_f}_\theta(\rho)$  }\label{Fig}	
\end{figure}

Looking at the SSRT $\check{S_f}$ of centrally symmetric objects centred in $f$ in Fig.4, we can notice the presence of reflection symmetry, which definition can be found in\cite{Nyg}, over the axe $\rho=0$ in the SSRT sinograms. This symmetry is the result of the reflection symmetry of the SSRT projections, $\check{S_f}_\theta(\rho)$, as seen in the same figure. In the following and for the sake of object central symmetry verification, we show some important SSRT characteristics.

\begin{proposition}
The SSRT is injective, which means that for two  $L^1(\mathbb{R}^2)$ measurable functions $f$ and $k$, $\check{S_f}=\check{S_k},\ \implies \ f=k $
\end{proposition}
\begin{proof}
Let be $f$ and $k$ two measurable $L^1 (\mathbb{R}^2)$ functions. According to \eqref{ssrt_conv} we can write $\check{S_f}$=$\check{S_k}\implies\check{R_f}\circledast g(\sigma)=\check{R_k}\circledast g(\sigma)$. Linearity of convolution provides that $\check{R_f}\circledast g(\sigma)=\check{R_k}\circledast g(\sigma)\implies (\check{R_f}-\check{R_k})\circledast g(\sigma)=0$. Let $\mathcal{F}$ denote the Fourier transform, then $\mathcal{F}$ is injective on $L^1$. The amount $\check{R_f}-\check{R_k}$ being a $L^1$ measurable function as RT is $L^1$ function \cite{symp}, then  $(\check{R_f}-\check{R_k})\circledast g(\sigma)=0 \implies \mathcal{F}( (\check{R_f}-\check{R_k})\circledast g(\sigma))=0$. The last expression being equal to $\mathcal{F}(\check{R_f}-\check{R_k})\mathcal{F}(g)$ and $\mathcal{F}(g)$ being  a Gaussian function up to a multiplicative constant which cannot, therefore, be zero ($\mathcal{F}(g)>0$ everywhere), thus $\mathcal{F}(\check{R_f}-\check{R_k})\mathcal{F}(g)=0 \implies \ \mathcal{F}(\check{R_f}-\check{R_k})=0$. Since $\check{R_f}-\check{R_k}$ is $L^1 (\mathbb{R}^2)$ function, the inverse of $\mathcal{F}(\check{R_f}-\check{R_k})$ exists and is zero only if $\check{R_f}-\check{R_k}=0$. As $\check{R_f}-\check{R_k}=0 \implies\check{R_f}=\check{R_k}\implies f=k$, because RT is injective \cite{Sigurdur}, consequently $\check{S_f}=\check{S_k} \implies f=k$ and hence, the SSRT is injective. 
\end{proof}
\begin{proposition}
For each $L^1(\mathbb{R}^2)$  measurable function $f$, $f(x,y)=f(-x,-y)\iff\forall$ $\theta$$ \ \in [0$\ $\pi]$, $\forall$$\rho \ \in [-\rho_m +\rho_m]$  $\check{S_f}_\theta(\rho)=\check{S_f}_\theta(-\rho)$.
\end{proposition}
 
\begin{proof}
1/ We start with the necessary condition. Let $f$ be a $L^1(\mathbb{R}^2)$ measurable function and let $\check{S_f}_\theta(\rho)$ be the SSRT projection of direction $\theta$. The projection $\check{S_f}_\theta(\rho)$ is a curve varying with the variable $\rho$, as seen in Fig.4. Thus, for a given $\theta$, $f(x,y)=f(-x,-y)\implies\check{S_f}_\theta(\rho)= \int_{\cal X}\int_{\cal Y}f(x,y)G_{\theta,\rho,\sigma}(x,y)dxdy=\int_{\cal X}\int_{\cal Y}f(-x,-y)G_{\theta,\rho,\sigma}(x,y)dxdy$, where $G_{\theta,\rho,\sigma}(x,y)$ is the directional Gaussian used in \eqref{ssrt}. By operating the variable change $x=-u$ and $y=-v$ and provided that $-u$ and $-v$ vary in the same support as $u$ and $v$ then, $\check{S_f}_\theta(\rho)= \int_{\cal U}\int_{\cal V}f(u,v)G_{\theta,\rho,\sigma}(-u,-v)dudv$. It is easy to check that $f(x,y)=f(-x,-y)\implies G_{\theta,\rho,\sigma}(-u,-v)=G_{\theta,-\rho,\sigma}(u,v)$. Consequently, we can write that $\check{S_f}_\theta(\rho)=\int_{\cal U}\int_{\cal V}f(u,v)G_{\theta,-\rho,\sigma}(u,v)dudv=\check{S_f}_\theta(-\rho)$. It results that $f(x,y)=f(-x,-y)\implies \check{S_f}_\theta(\rho)=\check{S_f}_\theta(-\rho)$, regardless $\theta$ and $\rho$.

2/ Considering now the sufficient condition. Let $f$ be a $L^1(\mathbb{R})$  measurable function. So, $\check{S_f}_\theta(\rho)=\check{S_f}_\theta(-\rho), \forall (\theta,\rho) \implies \int_{\cal X} \int_{\cal Y}f(x,y)G_{\theta,\rho,\sigma}(x,y)dxdy=\int_{\cal X} \int_{\cal Y}f(x,y)G_{\theta,-\rho,\sigma}(x,y)dxdy, \forall (\theta,\rho)$. Operating the following variable change $x=-u$ and $y=-v$ in the expression of $\check{S_f}_\theta(\rho)$ yields, $\check{S_f}_\theta(\rho)=\int_{\cal U} \int_{\cal V}f(-u,-v)G_{\theta,\rho,\sigma}(-u,-v)dudv$. Knowing that $G_{\theta,\rho,\sigma}(-u,-v)=G_{\theta,-\rho,\sigma}(u,v)$ provides $\int_{\cal U} \int_{\cal V}f(-u,-v)G_{\theta,-\rho,\sigma}(u,v)dudv$=
$\int_{\cal X} \int_{\cal Y}f(x,y)G_{\theta,-\rho,\sigma}(x,y)dxdy,  \forall (\theta,\rho)$. The function $f$ being a $L^1(\mathbb{R}^2)$  measurable  function and the SSRT being injective, as previously demonstrated, makes the last equality true only if $f(x,y)=f(-x,-y)$. Hence, $\check{S_f}_\theta(\rho)=\check{S_f}_\theta(-\rho),  \forall (\theta,\rho) \implies f(x,y)=f(-x,-y). $

\small From 1/ and 2/ $f(x,y)=f(-x,-y) \iff \forall\theta \in [0$\ $\pi]$, $\forall  \rho \in [-\rho_m +\rho_m]\check{S_f}_\theta(\rho)=\check{S_f}_\theta(-\rho)$.\normalsize
\end{proof}	

At this stage, the problem of central symmetry measurement is going to be moved from a 2D domain (the image $f$) to a 1D domain (the projections $\check{S_f}_{\theta}(\rho)$).To begin, let us consider that the binary object is centred in the image $f$ around its centroid $(x_c,y_c)=(0,0)$. To this end, we exploit, the \textbf{proposition 4.2} that stipulates that if for each direction $\theta$, the projection $\check{S_f}_{\theta}(\rho)$, where $\rho\in[-\rho_{max} \ \rho_{max}]$, is reflectionally symmetric over the axis $\rho=0$, then $f(x,y)=f(-x,-y)$. To the aim of evaluating the reflection symmetry of the projection $\check{S_f}_{\theta}(\rho)$, the latter is compared to its reflected form over the axis $\rho=0$, noted $\check{S_f}_\theta^{rf}$, which means that $\check{S_f}_\theta^{rf}(\rho)=\check{S_f}_\theta(-\rho), \ \forall \rho$. This comparison is done via a difference measure carried out on the projection and its reflected version. If this measure is null then the two projections are the same and consequently $\check{S_f}_{\theta}(\rho)$ is reflectionally symmetric. Let us call this measure $D$ computed as the ratio of the mean value of the $m\%$ highest values of $|\check{S_f}_{\theta}(\rho)-\check{S_f}_\theta^{rf}|$ and $\check{S_f}_{\theta}(\rho)$ maximum value $M_{S_f}$. Let us designate the set consisting in the $m\%$ highest values of $|\check{S_f}_{\theta}-\check{S_f}_\theta^{rf}|$ by $d_m$, and its mean value by $\bar{d_m}$,  then $D$ is computed as 
\begin{equation}\label{D}
D(\check{S_f}_{\theta},\check{S_f}_\theta^{rf})=\frac{\bar{d_m}}{M_{S_f}}
\end{equation}
When the centrally symmetric object is centred in the image $f$ and the origin of axes $(0,0)$ is the image center, then $f(x,y)=f(-x,-y)$ and the center of orientation coincides with the object centroid $(x_c,y_c)=(0,0)$. This leads to reflection symmetries of SSRT projections over the axes  $\rho^{(x_c,y_c)}_{\theta}=x_c\cos\theta+y_c\sin\theta=0, \ \forall \ \theta$. However, objects are rarely centred in the image. Hopefully and thanks to the shifting property of the SSRT, when the object is shifted in the image, each projection $\check{S_f}_{\theta}$ undergoes a shift of amount equals to $\rho^{(x_c,y_c)}_{\theta}=x_c\cos\theta+y_c\sin\theta$ \cite{Nacer}. Consequently, a non-centred centrally symmetric object SSRT projections shapes do not change compared to the ones of the same object centred in the image. in fact, parts of SSRT projections of the mentioned object keep, fortunately, a reflection symmetry over the axes $\rho^{(x_c,y_c)}_{\theta}$, as we can see in Fig.5(a). Consequently, a circular shift operation on each projection of an amount equals to $-\rho^{(x_c,y_c)}_{\theta}=-(x_c\cos\theta+y_c\sin\theta)$ makes all projections symmetric over the axis $\rho=0$ as illustrated, in Fig.5(b). Thus, to apply the proposed reflection symmetry verification method on SSRT projections, the latter must be subjected to such translations to make them reflectionally symmetric over the axis $\rho=0$ if the object has central symmetry.
 \begin{figure}[h]
	\centering
	\includegraphics[width=110mm]{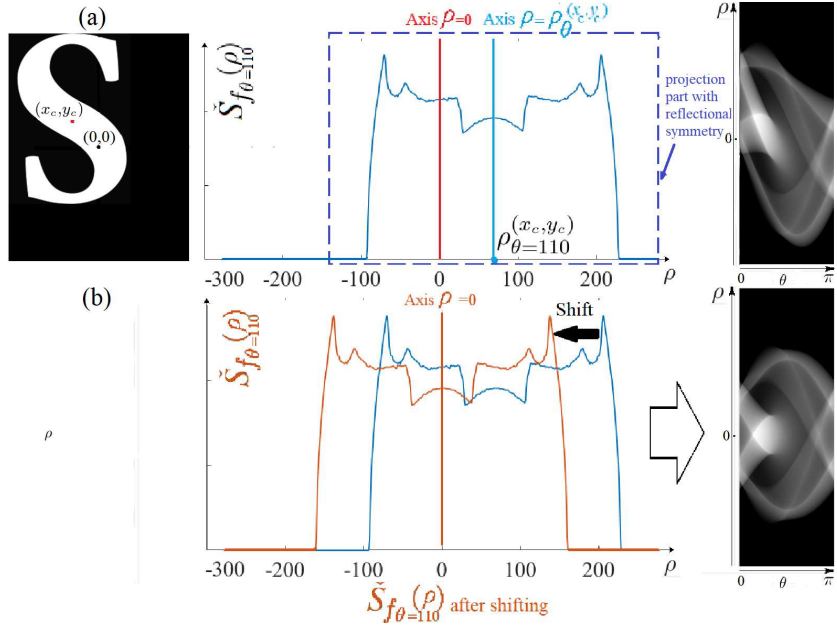}
	\caption{From left to right (a): Centrally symmetric object not centred in the image, a randomly chosen direction projection with mirror symmetry of a part of the projection over the axis $\rho^{x_c,y_c}_{\theta}$ and its SSRT sinogram. (b): The same projection (in blue) and its shifted version (orange colored) with mirror symmetry of the projection and the resulting SSRT sinogram after subjecting all projections to shift operations.}\label{Fig}	
\end{figure}

\subsection{Sampling the projections for central symmetry measurement}

Central symmetry measurement requires the exploitation of a number of $\check{S_f}$ projections. We provide in the following how we have chosen this number. 
\begin{proposition}
	For each $L^1(\mathbb{R}^2)$ measurable function $f$, such as $f(x,y)=f(-x,-y)$, then the reconstructed $\hat{f}$ with $\check{S_f}$ fulfils $\hat{f}(x,y)=\hat{f}(-x,-y)$,  whatever the number of used projections 
\end{proposition}
\begin{proof}
	Let $b_R$ be the back-projection of $\check{R_f}$. It is given by \cite{Sigurdur}, \small $b_R(x,y)=1/\pi\int_0^ \pi \check{R_f}(x\cos\theta+y\sin\theta,\theta)\ d\theta=1/\pi\int_0^\pi \check{R_f}(\rho,\theta)\ d\theta$\normalsize, as $x\cos\theta+y\sin\theta$ is nothing but $\rho$. With a finite number of equispaced views, $b_R(x,y)$ can be
	approximated by the summation $b_R(x,y)=1/M \sum_i^M \check{R_f}(\rho,\theta_i)$ \cite{Ling}, where $M$ is the number of the used projections. Similarly to $b_R$, let $b_S$ be the back-projection for a finite number of projection $M$ of $\check{S_f}$. Then, $b_S(x,y)=1/M\sum_i^M \check{S_f}(\rho,\theta_i)$. As $\check{S_f}(\rho,\theta_i)$ expresses the projection, then $\check{S_f}(\rho,\theta_i)=\check{S_f}_{\theta_i}(\rho)$. Moreover, as $f(x,y)=f(-x,-y)\iff \check{S_f}_{\theta_i}(\rho)=\check{S_f}_{\theta_i}(-\rho) \ \forall \ \rho \ \theta_i$, then it follows that $\check{S_f}(x\cos\theta_i+y\sin\theta_i,\theta_i)=\check{S_f}(-x\cos\theta_i-y\sin\theta_i,\theta_i)$ whatever $\theta_i$ and consequently, $b_S(x,y)=b_S(-x,-y)$ whatever $M$. Moreover, $b_S$ can be written by introducing the relationship between $\check{R_f}$ and $\check{S_f}$ in \eqref{ssrt_conv}, i.e $b_S(x,y)=1/M \sum_i^M \sum_r g(r-\rho)\check{R_f}(\rho,\theta_i)=\sum_r g(r-\rho)1/M \sum_i^M \check{R_f}(\rho,\theta_i)$. We have, therefore, $b_S(x,y)=\sum_r g(r-\rho)b(\rho,\theta_i)$. Consequently, $b_S=b_R\circledast g$. Furthermore, $b_R=\hat{f}\circledast h$, where $h=\frac{1}{\sqrt{x^2+y^2}}$ \cite{Sigurdur}. It follows that, $b_S=\hat{f}\circledast h\circledast g$ and therefore, $\hat{F}(w)=B(w)_S/\left(H(w)G(w)\right)$ for $G(w)\ne 0$ and $H(w)\ne 0$, with $B_S(w)$, $H(w)$, $\hat{F}(w)$ and $G(w)$ the Fourier Transform of $b_S$, $h$, $\hat f$ and $g$, respectively. In addition, $B(w)_S$ and $H(w)$ are real valued function and even in terms of $w$ because $h$ is centrally symmetric, by its expression and so is  $b_S$ whatever the projections number $M$. Concerning $G(w)$ it is equal to $e^{-2\pi^2\sigma^2w^2}$ as shown in \cite{Nacer} and is, therefore, real valued function and even in terms of $w$. Consequently, $\hat F(w)$ is a real even function in terms of $w$ because the product/quotient of two real even functions yields a real even function leading to $\hat f(x,y)=\hat f(-x,-y)$ and this holds whatever $M$.
\end{proof}

\begin{algorithm}
	\DontPrintSemicolon
	\SetKwFor{For}{for}{do}{end~for}
	\KwIn{Binary image $f$, the scale $ \sigma_{sym}$, threshold $\epsilon$, $\rho$ and $\theta$ steps} 
	
	\KwOut{$Sym$}
	Compute $\sigma$, $\check{R_f}$, $g(\sigma)$, $\check{S_f}$ with \eqref{ssrt_conv} and  $(\hat \theta, \hat \rho)= \underset{(\theta , \rho ) }{\mathrm{argmax}} \ \check{S_f} $
	
	Compute $g(\sigma_{sym})$, $\check{S_f}$, consider $\check{S_f}_{\theta_1}$, shift if necessary and reflect it
	
	Compute $D$ \eqref{D} for $\check{S_f}_{\theta_1}$
	
	\If { $ D \leq \epsilon$}{
		
		Consider $\check{S_f}_{\theta_2}$, shift if necessary and reflect it, compute $D$ \eqref{D} for $\check{S_f}_{\theta_2}$
		
		\If { $ D \leq \epsilon$}{
			
			Consider $\check{S_f}_{\theta_3}$, shift if necessary and reflect it, compute  $D$\eqref{D} for $\check{S_f}_{\theta_3}$
			
			\If { $ D \leq \epsilon$}{
				
				$Sym=1$
				
			}
			\Else {
				$Sym=0$
			}
			
		}
		\Else {
			$Sym=0$
		}
		
	}
	\Else {
		$Sym=0$
	}
	\caption{Central symmetry checking}
	\label{alg:alg1}
\end{algorithm} 
It results that using one reflectionnaly symmetric projection to construct $\hat f$ guarantees the central symmetry of the latter, but does not guarantee, unfortunately, the central symmetry of $f$. Indeed, if this projection of which orientation, says $\gamma$, corresponds to the orientation of an existing reflection symmetry in $f$, which also produces a reflection symmetry in the projection $\check{R_f}_{\gamma}$, as seen in \cite{Nyg} and subsequently in $\check{S_f}_{\gamma}(\rho)$, then, $\hat f$ will be centrally symmetric even if $f$ has only one reflection symmetry. To avoid this situation, we propose to sample the SSRT space into three equispaced projections, $\check{S_f}_\theta$, $\check{S_f}_{\theta+\pi/3}$ and $\check{S_f}_{\theta+2\pi/3}$. Concerning the choice of the angle $\theta$, it is going to be related the orientation of principal axes of inertia. In fact, these axes are used to characterize dispersion of bodies by representing the spatial distribution of their
mass \cite{Liu}. Furthermore, the axis of reflectional symmetry of a flat object is going to be one of its inertia axes. So, to avoid the axis of reflection symmetry orientation, if it exists, we propose to deviate from such orientation by taking $\theta$ equals to $\hat \theta$ computed in Sect.3, added to $\Delta \theta$. So, if we call $\theta_1=\hat \theta+\Delta \theta$, $\theta_2=\theta_1 +\pi/3$ and  $\theta_2=\theta_1 +2\pi/3$, then the symmetry is ascertained if the measure $D$ computed for the three orientation is below a threshold $\epsilon$.

It is worth to note, that the SSRT scale parameter $\sigma_{sym}$ used to compute symmetry is not the same as the one used to compute the axis of inertia. In fact, the last one has as purpose to produce smooth projections curves which makes the method less sensitive to turbulences that may be caused by non-smooth object edges or by noise.
The operations related to the central symmetry measurement can be summarized in Algorithm 1, where the output is a boolean variable $Sym$ equals to 1 if the object is centrally symmetric and to 0 if not.

\section{Experiments}
In order to evaluate the inertia axes estimation with the SSRT, we compare inertia axes obtained with the geometric moments and those computed with the SSRT maxima. In order to guarantee the obtention of one maximum on the SSRT, the scale parameter $\sigma$ should be tuned correctly. This condition is fulfilled by setting it to the biggest distance separating two processed object points, for binary images. For gray-scale ones, set $\sigma$ to the diagonal length of the image has provided the desired results. We can see in Fig.6 examples of patterns extracted from the dataset in \cite{data}, on which are depicted red lines corresponding to SSRT maxima related to each image and the ones in dashed cyan representing the geometric moments-based principal inertia axes. We can see that, in each case, the two lines match. The same remarks can be done for real images in Fig.7, where the geometric moments-based axes of inertia and SSRT-based axes of inertia have been computed directly on gray scale images. We can see, in this case, that the two axes coincide. We can remark, however, that in some cases (white fox, swans, white lion) there is a slight deviation of the SSRT line from the inertia axis. This is due to the fact that SSRT is estimated on a sampled space. This makes $\rho$ and $\theta$ taking values only in the SSRT discrete parameters space, unlike the axis of inertia computed with the geometric moments of which slope and position are freed from such constraints 
\begin{figure}[h]
	\centering
	\includegraphics[width=120mm]{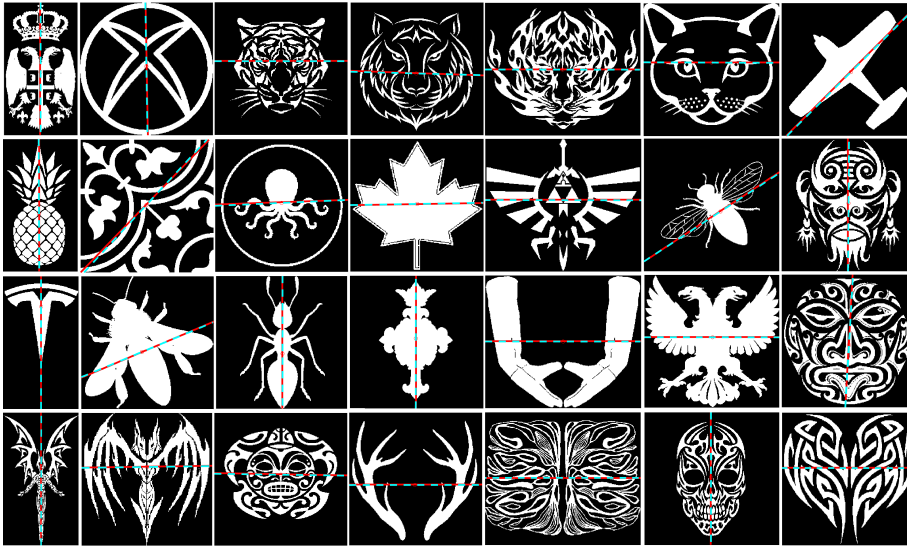}
	\caption{Geometric moments-based main axis of inertia in cyan dashed line and SSRT maximum-based line in red, computed on images}\label{Fig}	
\end{figure}
\begin{figure}[h]
	\centering
	\includegraphics[width=110mm]{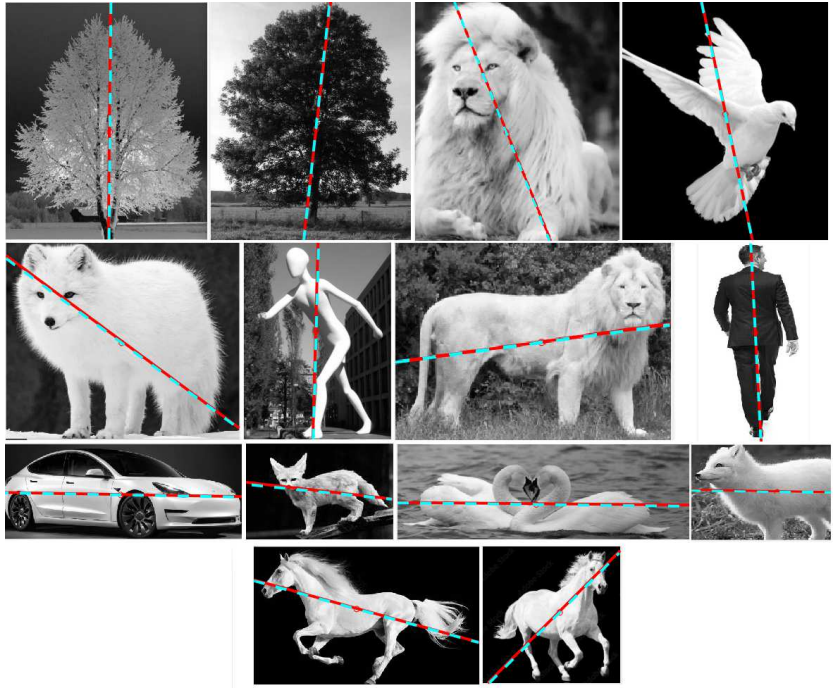}
	\caption{Geometric moments-based main axis of inertia and SSRT based axis of inertia computed on real images}\label{Fig}	
\end{figure}
 \begin{figure}[h]
	\centering
	\includegraphics[width=100mm]{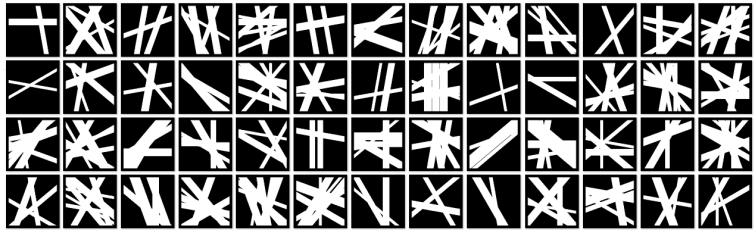}
	\caption{A sample of the randomly generated images}\label{Fig}	
\end{figure}
\begin{figure}[h]
	\centering
	\includegraphics[width=90mm]{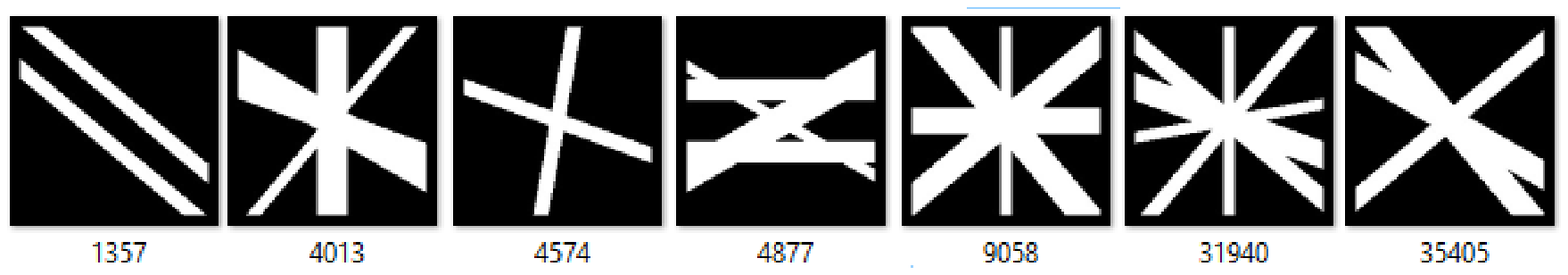}
	\caption{A sample of images with centrally symmetric patterns extracted from generated dataset.}\label{Fig}	
\end{figure}

Furthermore, to test our proposed central symmetry measurement method, a set of 100 000 images, composed of generated bars of randomly chosen numbers, widths, orientations and positions, are used. A sample of this dataset is depicted in Fig.8. On each image, inertia axis is computed with the SSRT maximum, followed by computing the SSRT with $\sigma_{sym}$ and finally, three projections are selected according to the inertia axis orientation and then shifted if necessary, as seen in Set.4. We point out here, that three parameters have to be set. The first one is $\epsilon$, the difference measure threshold that defines how the SSRT projection and its reflected version are far from each other, the second one is $\Delta \theta$ the angle added to deviate from principal direction of inertia and finally $\sigma_{sym}$, the scale of the SSRT for symmetry measurement, which should be a trade-off between projection smoothing and the fidelity guarantee of object representation. These three parameters are chosen to be equal to 0.03 for $\epsilon$, as images are composed of perfect binary bars and we want the SSRT projection and its reflection form comparison to be an accurate one, to 5° for the angle deviation $\Delta \theta$ which seemed to be adequate for the experiments, and to 1 for the scale $\sigma_{sym}$. Applying our method on this collection of dataset, has divided it into a group of 730 images containing centrally symmetric objects and another one consisting of 99270 images. A sample of the first group is given in Fig.9. To evaluate the performance of our method, an evaluation procedure is performed on its outcomes by comparing them to a ground truth. The latter is obtained by dividing the dataset into two groups, on the basis of the direct application of the central symmetry definition of an object, given in Set 4. This division is performed as follows; for every image $f$ in the dataset, another one, $f_r$, is generated by rotating the object, noted $f_O$ around its centroid by $\pi$. If the outcome of the measure $E_m=area(f_O^{sub}) /area(f_O)$, where $f_O^{sub}$ is the object created by the operation $|f-f_r|$, is less then a threshold $t$, where $t$ is set to 0.1, then the object in $f$ is said to be centrally symmetric and assigned to central symmetry group. It is worth mentioning that the direct application of object central symmetry definition to divide the dataset and obtain the ground truth is made possible by the fact that the latter is composed of perfect binary structures. Afterwards, the ground truth so obtained, is compared to the outcomes of our method application. Hence, all images found to be centrally symmetric by our method have been also approved as centrally symmetric by the reference dataset division, except two (2) of them while 34 images detected as non-centrally images by our method have been found to be centrally symmetric by the reference division. If we look in Fig.10 (b), we will see that images "erroneously" assigned by our method to centrally symmetric group, have a missed threadlike part inside their objects, which makes them not completely centrally symmetric even if their overall shape seem to be coarsely symmetric. However, this could just express insensitivity to impulse noise of the proposed method as we can see later. Regarding the images affected "erroneously" to non-centrally symmetric group as images in Fig.10 (a), we can see visually that the are not perfectly centrally symmetric. In turns out that, assigning an image to a particular group is handled by $\epsilon$ tuning. Indeed, the more $\epsilon$ is close to zero the more the affectation operation is rigorous. At the light of these numerical results, and if we consider the processes of dividing the dataset as a binary classification, then the precision rate $A$ of our classification where $A=\frac{TP}{TP+FP}$ with $TP$ the true positive number i.e. the number of images with centrally symmetric objects that have been correctly assigned as such and $FP$ are the false negative number which is the number of images that have been erroneously assigned to the centrally symmetric group. Then, $A=\frac{728}{730}=0.997$ which represents a satisfying result. It is important to mention that the precision $A$ could slightly change by increasing $\epsilon$ to make the symmetry measurement more flexible or decreasing $t$ to make the ground truth creation more severe. To finish with this dataset, corrupting centrally symmetric images by impulse noise with density equal to 0.1, has allowed to test the robustness of the method against noise. To face noise in images we have increased $\sigma_{sym}$ to 10 and observe the outcomes. Corrupting with noise images composed of filiform bars have subjected the SSRT projections to important modifications leading to partial lost of their reflection symmetries as shown on Fig.11, and therefore, the method fails in detecting their central symmetries in such images, unlike the other ones. It follows that the way we compare SSRT projections and their reflected versions should be adjusted to face projections behaviour when the image is subjected to noise, and then enhance robustness against noise for patterns of all possible shapes. We can see in Fig.12 the example of images in Fig.9 checked as centrally symmetric even with impulse noise.
\begin{figure}[h]
	\centering
	\includegraphics[width=143mm]{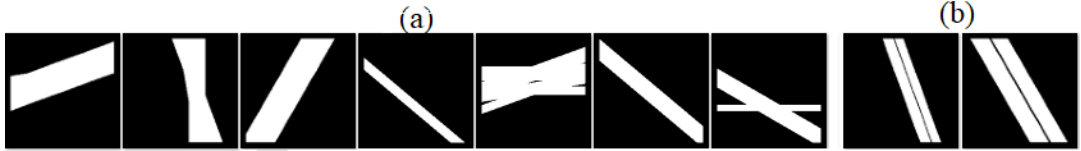}
	\caption{Results of evaluation. (a): Images erroneously assigned to non-centrally symmetric group (b): Images erroneously assigned to centrally symmetric group}\label{Fig}	
\end{figure}
\begin{figure}[h]
		\centering
	\includegraphics[width=140mm]{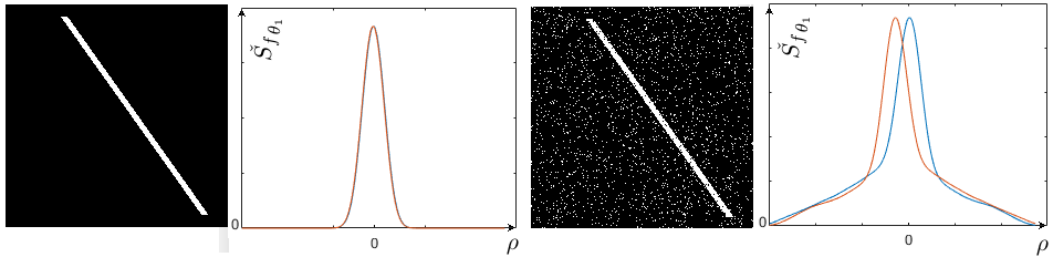}
	\caption{From left to right: Filifom bar, complete match between its SSRT projection $\check{S_f}_{\theta_1}$ and its reflected version $\check{S_f}^{rf}_{\theta_1}$, noisy filiform bar, mismatch between $\check{S_f}_{\theta_1}$ in bleu and its reflected version $\check{S_f}^{rf}_{\theta_1}$ in red}\label{Fig}	
\end{figure}
\begin{figure}[h]
	\centering
	\includegraphics[width=140mm]{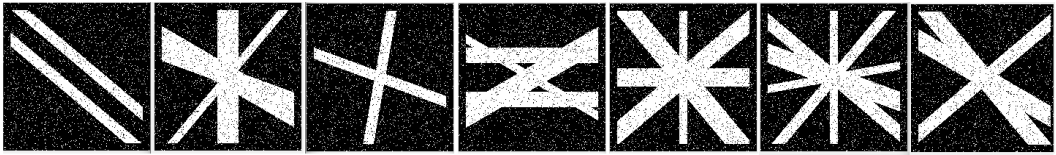}
	\caption{Example of centrally symmetric images corrupted with salt \& pepper noise and checked with the proposed method as centrally symmetric ones }\label{Fig}	
\end{figure}

The last experiment is carried out on the dataset in \cite{data}. This dataset is composed of 100 images containing patterns with several rotational symmetries and is displayed in Fig.13. We recall that an $n$ folds rotational symmetric object is an object that looks the same after being subjected to rotation around its centroid by $2\pi/n$, where $n\in \mathbb{N}$ \cite{Pei}. Hence, in this dataset, some of images are $2n$ folds rotational symmetric and hence, centrally symmetric and the others are $2n+1$ folds rotational symmetric and then, non centrally symmetric. Furthermore, these images are not binary ones, they must be thresholded. Consequently, here, $\epsilon$ is increased to 0.1 to deal with irregularities that may appear in the object after thresholding. From the 100 images belonging to the mentioned dataset, the proposed method has permitted to distinguish, successfully, between the centrally symmetric ( $2n$ folds rotational symmetric) and non-centrally symmetric images ($2n+1$ folds rotational symmetric). The latter are depicted in Fig.14.

\begin{figure}[h]
	\centering
	\includegraphics[width=140mm]{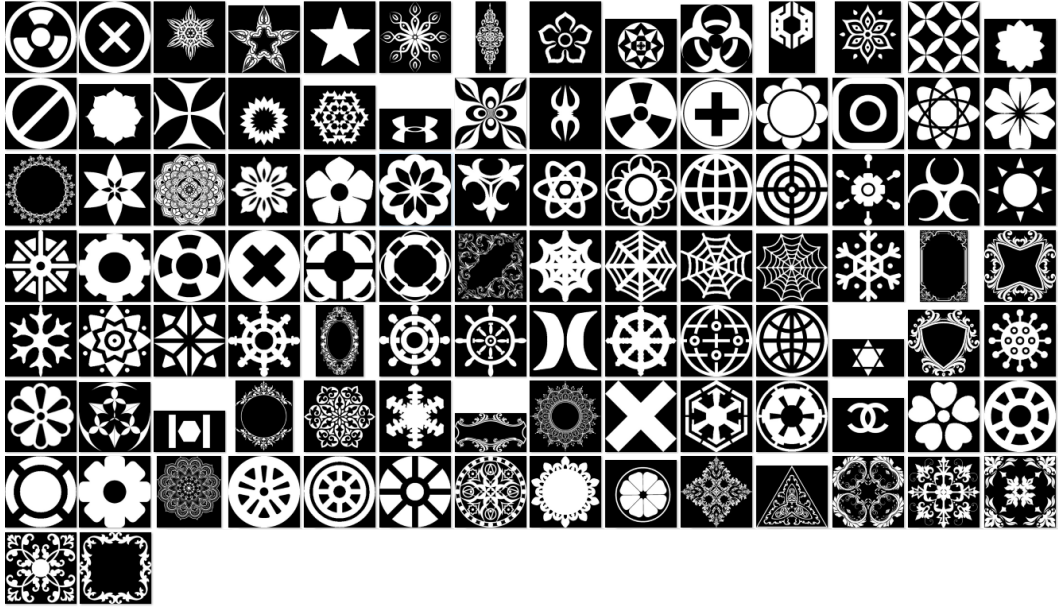}
	\caption{Images of the dataset in \cite{data}}\label{Fig}	
\end{figure}
\begin{figure}[h]
	\centering
	\includegraphics[width=140mm]{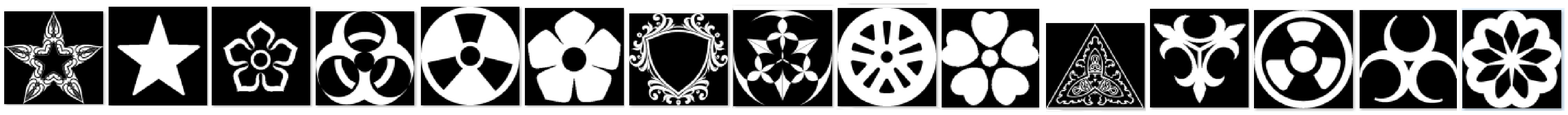}
	\caption{Extracted images consisting in non-centrally symmetric objects of Fig.12}\label{Fig}	
\end{figure}
  \section{Conclusion}
In this paper, we have proposed to investigate the ability of the Scale Space Radon Transform to provide the main axis of inertia by means of its maximum, when the corresponding scale is chosen correctly. Mathematical expressions of the parameters of the SSRT for a line, obtained by derivation, have shown to give the same expressions of the line parameters of the main axis of inertia computed with the geometric moments. Furthermore, experimental results have shown that axes of inertia and the SSRT maxima-based lines computed on gray scale and binary images are almost overlapped, which indicates that they match. In addition, the proposed central symmetry measurement method tested on two datasets, has shown its effectiveness, by permitting, therefore, to pick out the centrally symmetric objects from the other ones. However, investigating a more effective difference measure to compare the SSRT projections and their reflection versions and increase, consequently, the robustness to noise will be appreciated. Moreover, the application of the method in 3D, will be the subject of future works to detect automatically Centro-symmetric structures in volumes.

\section{Appendix}
The exploited approximation of $g(z)=e^{-z^2/2\sigma^2}$ with Maclaurin serie being limited to $n=1$, we compute here the error of this approximation in terms of serie remainder. So, let us consider the serie remainder for n=1, $r_{1}$ which is equal to $\sum^\infty_{n=2}(-1)^n\frac{1}{n!}\left(\frac{z^2}{2\sigma^2}\right)^n$. It is known for alternating serie $\sum_{0}^{\infty} (-1)^n a_n$ that fulfils the convergence conditions (1) and (2) stated in Set.3, that the remainder $r_n$ satisfies $|r_n|\leq a_{n+1}$. Consequently, $a_{n+1}$ being equal to $\frac{1}{2}(\frac{z^2}{2\sigma^2})^2$ for n=1, it turns out that $|r_1|\leq \frac{1}{2}(\frac{z^2}{2\sigma^2})^2$. The surface in Fig.\ref{Fig15} shows the evolution of $\frac{1}{2}(\frac{z^2}{2\sigma^2})^2$ with respect to $z$ and $\sigma$. The approximation error being upper bounded by $\frac{1}{2}(\frac{z^2}{2\sigma^2})^2$, it drops considerably and quickly, as its upper bound does, when $\sigma$ increases. Since the scale space parameter $\sigma$ is directly related to the object size or to the image size, as previously seen, its amount exceeds for sure 1, and will certainly rise to over 25, the maximum value of this parameter in the remainder upper bound evolution figure. For example for $z=0.25$ and $\sigma=20$, $\frac{1}{2}(\frac{z^2}{2\sigma^2})^2=3\times 10^{-9}$, which ascertains the good approximation we have chosen.
\begin{figure}[h]
	\centering
	\includegraphics[width=145mm]{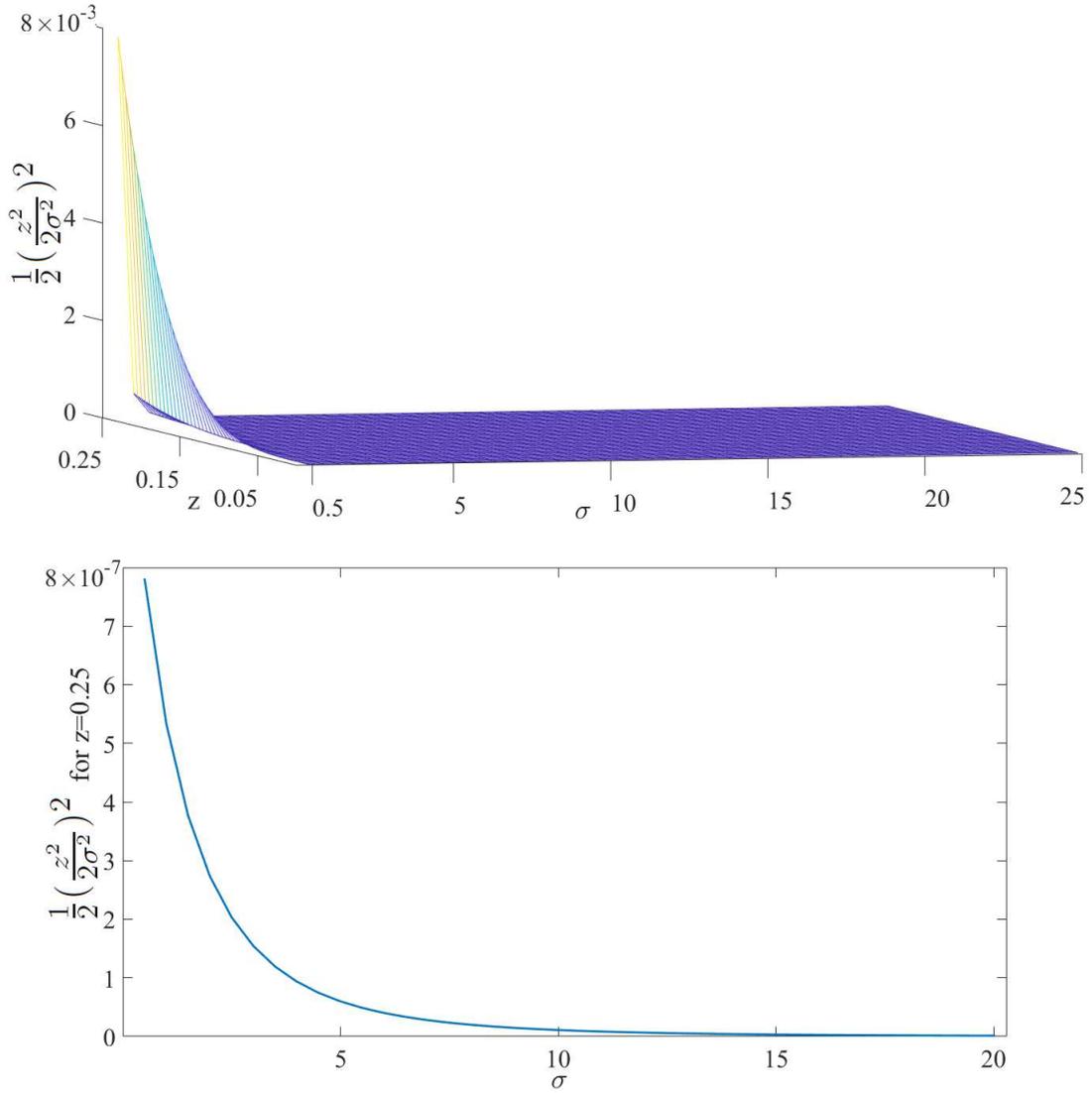}
	\caption{From top to bottom: Evolution of $\frac{1}{2}(\frac{z^2}{2\sigma^2})^2$ with the variable $z$ and the SSRT scale space parameter $\sigma$. Evolution of the approximation error upper bound $\frac{1}{2}(\frac{z^2}{2\sigma^2})^2$ when $z$ is set to 0.25. }\label{Fig15}	
\end{figure}
\end{document}